\documentclass{tlp}

\usepackage[utf8]{inputenc}
\usepackage{amsmath}
\usepackage{amssymb}
\usepackage{microtype}
\usepackage{url}
\usepackage{stmaryrd}
\usepackage{listings}
\RequirePackage{bm}
\RequirePackage{textcomp}
\RequirePackage{upgreek}

\RequirePackage{xcolor}
\makeatletter
\input pdf-trans
\newbox\qbox
\def\usecolor#1{\csname\string\color@#1\endcsname\space}
\newcommand\bordercolor[1]{\colsplit{1}{#1}}
\newcommand\fillcolor[1]{\colsplit{0}{#1}}
\newcommand\outline[1]{\leavevmode \def\maltext{#1}\setbox\qbox=\hbox{\maltext}\boxgs{Q q 2 Tr \thickness\space w \fillcol\space \bordercol\space}{}\copy\qbox }
\makeatother
\newcommand\colsplit[2]{\colorlet{tmpcolor}{#2}\edef\tmp{\usecolor{tmpcolor}}\def\tmpB{}\expandafter\colsplithelp\tmp\relax \ifnum0=#1\relax\edef\fillcol{\tmpB}\else\edef\bordercol{\tmpC}\fi}
\def\colsplithelp#1#2 #3\relax{\edef\tmpB{\tmpB#1#2 }\ifnum `#1>`9\relax\def\tmpC{#3}\else\colsplithelp#3\relax\fi
}
\bordercolor{black}
\fillcolor{white}
\def\thickness{.3}

\newcommand{\next}{\text{\rm \raisebox{-.5pt}{\Large\textopenbullet}}}  \newcommand{\previous}{\text{\rm \raisebox{-.5pt}{\Large\textbullet}}}  \newcommand{\wnext}{\ensuremath{\widehat{\next}}}
\newcommand{\wprevious}{\ensuremath{\widehat{\previous}}}
\newcommand{\alwaysF}{\ensuremath{\square}}
\newcommand{\alwaysP}{\ensuremath{\blacksquare}}
\newcommand{\eventuallyF}{\ensuremath{\Diamond}}
\newcommand{\eventuallyP}{\ensuremath{\blacklozenge}}
\newcommand{\until}{\ensuremath{\mathbin{\mbox{\outline{$\bm{\mathsf{U}}$}}}}}
\newcommand{\release}{\ensuremath{\mathbin{\mbox{\outline{$\bm{\mathsf{R}}$}}}}}

\newcommand{\since}{\ensuremath{\mathbin{\bm{\mathsf{S}}}}}
\newcommand{\trigger}{\ensuremath{\mathbin{\bm{\mathsf{T}}}}}
\newcommand{\finally}{\ensuremath{\mbox{\outline{$\bm{\mathsf{F}}$}}}}
\newcommand{\initially}{\ensuremath{\bm{\mathsf{I}}}}

\mathchardef\mhyphen="2D

\newcommand{\intervco}[2]{\ensuremath{[#1..#2)}}
\newcommand{\intervoc}[2]{\ensuremath{(#1..#2]}}

\newcommand{\rangeco}[3]{\ensuremath{#1 \in \intervco{#2}{#3}}}
\newcommand{\rangeoc}[3]{\ensuremath{#1 \in \intervoc{#2}{#3}}}

 \providecommand{\logfont}{\textrm}

\newcommand{\HT}{\ensuremath{\logfont{HT}}}

\newcommand{\HTC}{\ensuremath{\logfont{HT}_{\!c}}}

\newcommand{\LTL}{\ensuremath{\logfont{LTL}}}

\newcommand{\THT}{\ensuremath{\logfont{THT}}}
\newcommand{\THTf}{\ensuremath{\THT_{\!f}}}

\newcommand{\TEL}{\ensuremath{\logfont{TEL}}}

\newcommand{\THTC}{\ensuremath{\THT_{\!c}}}

\newcommand{\tuple}[1]{\ensuremath{\langle #1 \rangle}}
\newcommand{\Htrace}{\ensuremath{\mathbf{H}}}
\newcommand{\Ttrace}{\ensuremath{\mathbf{T}}}
\newcommand{\M}{\ensuremath{\mathbf{M}}}

\providecommand{\Underscore}{\textunderscore}

\lstdefinelanguage{clingo}{basicstyle=\ttfamily,keywordstyle=[1]\bfseries,keywordstyle=[2]\bfseries,keywordstyle=[3]\bfseries,showstringspaces=false,literate={_}{\Underscore}1 {\%\%}{}0,escapeinside={\#(}{\#)},alsoletter={\#,\&},keywords=[1]{not,from,import,def,if,else,elif,return,while,break,and,or,for,in,del,and,class,with,as,is,yield,async},keywords=[2]{\#const,\#show,\#minimize,\#base,\#theory,\#count,\#external,\#program,\#script,\#end,\#heuristic,\#edge,\#project,\#show,\#sum},morecomment=[l]{\#\ },morecomment=[l]{\%\ },morestring=[b]",stringstyle={\itshape},commentstyle={\color{darkgray}}}

\lstdefinelanguage{clingcon}[]{clingo}{morekeywords={&dom,&sum,&nsum,&diff,&disjoint,&distinct,&minimize,&maximize,&show}}
\lstdefinelanguage{fclingo}[]{clingo}{morekeywords={&sum,&sus,&in,&df,&min,&max,&show}}
\lstdefinelanguage{clingodl}[]{clingo}{morekeywords={&diff}}

\lstdefinelanguage{python}{basicstyle=\ttfamily,keywordstyle=[1]\bfseries,showstringspaces=false,literate={_}{\Underscore}{1},escapeinside={\#(}{\#)},alsoletter={\#,\&},keywords=[1]{not,from,import,def,if,else,elif,return,while,break,and,or,for,in,del,and,class,with,as,is,yield,async},morecomment=[l]{\#\ },morestring=[b]",stringstyle={\itshape},commentstyle={\color{darkgray}}}

\newtheorem{definition}{Definition}
\newtheorem{proposition}{Proposition}

\newenvironment{proofof}[1]{\noindent {\bf Proof of #1.}}{\bigskip}

\newcommand{\var}{\ensuremath{\mathcal{A}}}
\newcommand{\Vh}{\ensuremath{\boldsymbol{v}_h}}
\newcommand{\Vt}{\ensuremath{\boldsymbol{v}_t}}
\newcommand{\Vw}{\ensuremath{\boldsymbol{v}_w}}

\newcommand{\eqdef}{\ensuremath{\mathbin{\raisebox{-1pt}[-3pt][0pt]{$\stackrel{\mathit{def}}{=}$}}}}
\newcommand{\undefined}{\ensuremath{u}}

\newcommand{\den}[1]{\llbracket #1 \rrbracket}
\newcommand{\tbl}[1]{[#1]}

\newcommand{\intervo}[2]{[#1..#2)}
\newcommand{\rangeo}[3]{#1 \in \intervo{#2}{#3}}
\newcommand{\intervc}[2]{[#1..#2]}
\newcommand{\rangec}[3]{#1 \in \intervc{#2}{#3}}
\newcommand{\true}{\ensuremath{\mathtt{t}}}

\newcommand{\vars}{\ensuremath{\mathit{terms}}}
\newcommand{\df}{\ensuremath{\mathit{df}}}

\newcommand{\ST}[2]{\tau_{#1}(#2)} \newcommand{\TERMS}[1]{\mathcal{T}(#1)}
\newcommand{\ATOMS}[1]{\mathcal{A}(#1)}
\newcommand{\QHT}{\ensuremath{\logfont{QHT}_{\mathcal{F}}^{=}(<)}} \newcommand{\peq}{\ensuremath{\preceq}}
\newcommand{\pt}{\ensuremath{\prec}}

\newcommand{\HTCSAT}{\ensuremath{\models_{\scriptscriptstyle\HTC}}}
\newcommand{\THTCSAT}{\ensuremath{\models_{\scriptscriptstyle\THTC}}}

\begin{document}

\lefttitle{P. Cabalar, M. Di\'eguez, F. Olivier, T. Schaub and I. St\'ephan}
\jnlPage{1}{42}
\jnlDoiYr{1984}
\doival{10.1017/XYZ}

\title{\mbox{Towards Constraint Temporal Answer Set Programming}}
\begin{authgrp}
  \author{\sn{Pedro} \gn{Cabalar}}
  \affiliation{University of Corunna, Spain}
  \author{\sn{Mart\'{\i}n} \gn{Di\'eguez}}
  \affiliation{University of Angers, France}
  \author{\sn{Fran\c{c}ois} \gn{Olivier}}
  \affiliation{CRIL CNRS \& Artois University, France}
  \author{\sn{Torsten} \gn{Schaub}}
  \affiliation{University of Potsdam, Germany}\affiliation{Potassco Solutions, Germany}
  \author{\sn{Igor} \gn{St\'ephan}}
  \affiliation{University of Angers, France}
\end{authgrp}

\maketitle

\begin{abstract}
  Reasoning about dynamic systems with a fine-grained temporal and numeric resolution presents significant
  challenges for logic-based approaches like Answer Set Programming (ASP).
  To address this, we introduce and elaborate upon a novel temporal and constraint-based extension of the logic of
  Here-and-There and its nonmonotonic equilibrium extension, representing, to the best of our knowledge, the first
  approach to nonmonotonic temporal reasoning with constraints specifically tailored for ASP.
  This expressive system is achieved by a synergistic combination of two foundational ASP extensions:
  the linear-time logic of Here-and-There, providing robust nonmonotonic temporal reasoning capabilities, and
  the logic of Here-and-There with constraints, enabling the direct integration and manipulation of numeric constraints, among others.
  This work establishes the foundational logical framework for tackling
  complex dynamic systems with high resolution within the ASP paradigm.
\end{abstract}

\section{Introduction}
\label{sec:introduction}

Reasoning about action and change is crucial for understanding
how dynamic systems evolve over time and how actions influence those changes.
Representing dynamic systems with higher resolution, such as by employing finer time units and more precise numeric variables,
significantly increases the complexity of reasoning about them,
posing particular challenges for logic-based approaches like Answer Set Programming (ASP;~\citealp{lifschitz08b}).
To illustrate this, consider the following scenario:
``{\em
  A radar is positioned at the 400 km mark on a road with a speed limit of 90 km/h.
  A car is initially traveling at 80 km/h.
  At time instant 4, the car accelerates by 11.35 km/h.
  Subsequently, at time instant 6, it decelerates by 2.301 km/h.
  The problem is to determine whether the car will exceed the speed limit and thus incur a fine.}''
A closer look reveals several key numeric entities:
the car's position and speed,
the radar's position and speed limit, and
the car's acceleration and deceleration.

Our objective is to provide the logical foundations for extending ASP
to effectively model such scenarios involving changing numeric values over time.
To achieve this,
we integrate two key logical formalisms that extend ASP:
the  linear-time logic of Here-and-There (\THT; \citealp{agcadipescscvi20a}),
with its ability to handle nonmonotonic temporal reasoning, and
the logic of Here-and-There with constraints (\HTC;~\citealp{cakaossc16a}),
which allows us to incorporate and reason about numeric constraints, among others.
Building upon standard linear temporal logic,
\THT\ and its nonmonotonic equilibrium extension, \TEL,
enable the expression of sophisticated temporal behaviors in dynamic systems through features like inertia and default reasoning.
\HTC\ complements this by allowing us to directly reason with numeric and other constraints.
Furthermore, its equilibrium extension provides solid logical foundations for tackling nonmonotonic constraint
satisfaction problems,
elegantly handling situations with incomplete information by using default values.
With this motivation established, we now proceed to introduce the combined logical framework.

 \section{Temporal Here-and-There with Constraints}\label{sec:background}

The syntax of the logic \emph{Temporal Here-and-There with constraints} (\THTC)
relies on a signature~$\tuple{\mathcal{X},\mathcal{D},\mathcal{A}}$,
akin to constraint satisfaction problems (CSPs;~\citealp{dechter03a}).
Specifically,
$\mathcal{X}$ denotes a set of variables and~$\mathcal{D}$ represents the domain of values,
often identified with their corresponding constants.
The set $\mathcal{A}$ comprises \emph{temporal constraint atoms} (or simply \emph{atoms}), which are defined over temporal terms.

A \emph{temporal term} (or simply \emph{term}) is fundamental to \THTC\ and
is intended to represent the value of a variable at past, present, or future time points.
We represent such a term by an expression $\next^ix$ where $x\in\mathcal{X}$ and $i\in\mathbb{Z}$,
while overloading the temporal modal operator for ``next'', viz.\ $\next$.
The integer $i$ indicates a temporal offset:
a positive $i$ signifies $i$ steps forward in time to retrieve the value of $x$,
a negative $i$ signifies $|i|$ steps backward, and
$i=0$ refers to the value of $x$ in the current state.
For notational convenience,
we use $x$ and $\next x$ as shorthand for $\next^0x$ and $\next^1x$, respectively.
Furthermore, we overload the operator $\previous$ for ``previous'' to represent past variable values.
That is, $\previous^ix$ stands for $\next^{-i}x$ for offset $i$.
For example,
the constraint $\next x + \previous y \le z$ is equivalent to $\next^1 x +  \next^{-1} y \le \next^0 z$.

Accordingly, an $n$-ary temporal constraint atom $c\in\mathcal{A}$ is often represented as
\(
c(\next^{o_1} x_1, \dots, \next^{o_n} x_n)
\)
where $\next^{o_1} x_1, \dots, \next^{o_n} x_n$ are (possibly identical) temporal terms.
To further illustrate, the atom
$(\previous^3 x = 4)$ can be read as ``$x$ had value 4 three states ago'', and
$(\next x = x)$ as ``the value of $x$ in the next state is identical to its current value''.

Then, \emph{temporal constraint formulas} (or just \emph{formulas}) are defined as follows:
\begin{align*}
  \varphi & ::= c\in\mathcal{A} \mid
            \bot \mid \varphi\wedge\varphi \mid\varphi\vee\varphi \mid\varphi\to\varphi \mid
            \next    \varphi \mid \varphi\until\varphi \mid \varphi\release\varphi\mid
            \previous\varphi \mid \varphi\since\varphi \mid \varphi\trigger\varphi
\end{align*}
Connectives $\bot$, $\wedge$, $\vee$ and $\to$ are \emph{Boolean}, while
the remaining connectives are \emph{temporal} modalities.
We use outlined operators to refer to future modalities and solid ones for past modalities.
Accordingly, $\next$, $\until$, $\release$
represent the future modalities \emph{next}, \emph{until}, \emph{release}, respectively,
while $\previous$, $\since$, $\trigger$ reflect their past counterparts:
\emph{previous}, \emph{since}, \emph{trigger}.

We also define several derived operators, including the Boolean connectives
\(
\top \eqdef \neg \bot
\),
\(
\neg \varphi \eqdef  \varphi \to \bot
\),
\(
\varphi \leftrightarrow \psi \eqdef (\varphi \to \psi) \wedge (\psi \to \varphi)
\),
and the following temporal operators:
\begin{align*}
  \initially          & \eqdef \neg\previous\top              &
  \wprevious  \varphi & \eqdef \previous\varphi\vee\initially &
  \alwaysP    \varphi & \eqdef \bot\trigger\varphi            &
  \eventuallyP\varphi & \eqdef \top\since\varphi
  \\
  \finally            & \eqdef \neg\next\top                  &
  \wnext \varphi      & \eqdef \next\varphi\vee\finally       &
  \alwaysF \varphi    & \eqdef \bot\release\varphi            &
  \eventuallyF\varphi & \eqdef \top\until\varphi
\end{align*}
Specifically,
$\initially$ and $\finally$ refer to the ``initial'' and ``final'' state, respectively;
$\wprevious \varphi$ ($\wnext\varphi$) refers to  $\varphi$ at the previous (next) state in case that it exists;
$\alwaysP \varphi$ ($\alwaysF \varphi$) is read as  ``$\varphi$ has always been true'' (``$\varphi$ will always be true''),
and $\eventuallyP \varphi$ ($\eventuallyF \varphi$) is read as ``$\varphi$ has been true'' (``$\varphi$ will eventually be true'').
For instance, $\alwaysF(\next x > x)$ allows us to express informally that
``the value of variable $x$ is increasing over time''.

We define the semantics of \THTC\ in terms of \HTC\ traces, which are sequences of \HTC\ interpretations~\citep{cakaossc16a}.
We first present the necessary semantic concepts of \HTC\ before further elaborating on these traces.
\HTC\ relies on partial valuations, that is, functions $v : \mathcal{X} \to \mathcal{D} \cup \{ \undefined \}$,
where the range is augmented by the special value $\undefined \not \in \mathcal{D}$ representing ``undefined''.
We define $\mathcal{D}_{\undefined}\eqdef \mathcal{D} \cup \{ \undefined\}$ and
denote such valuations also by $v : \mathcal{X} \to \mathcal{D}_{\undefined}$.
Given two partial valuations $v$ and $v'$,
we define $v \sqsubseteq v'$ if $v(x) = d$ implies $v'(x) = d$
for all $x\in\mathcal{X}$ and $d \in \mathcal{D}$.
Similarly, we say that
$v = v'$ if $v \sqsubseteq v'$ and
$v'\sqsubseteq v$, and $v \sqsubset v'$ if $v \sqsubseteq v'$ and $v\neq v'$.
An \HTC\ interpretation is a pair $\langle v_h, v_t \rangle$ of valuations such that $v_h \sqsubseteq v_t$.

Building upon this,
we define a \emph{trace} $\boldsymbol{v}$ of length $\lambda$ as
a sequence $\boldsymbol{v}=(v_i)_{\rangeo{i}{0}{\lambda}}$ of partial valuations $v_i$ for $\rangeo{i}{0}{\lambda}$.
For instance,
the sequence
\(
\boldsymbol{v}'=
\left\{ x \mapsto 4          \right\} \cdot
\left\{ x \mapsto \undefined \right\} \cdot
\left\{ x \mapsto 5          \right\}
\)
represents a finite trace with three valuations (separated by `$\cdot$') that
assign different values to $x$ except at the second state where $x$ is undefined.
The $\sqsubseteq$ relation can be extended to cope with traces in the following way:
given two traces $\Vh = (v_{h,i})_{\rangeco{i}{0}{\lambda}}$ and $\Vt = (v_{t,i})_{\rangeco{i}{0}{\lambda}}$ of length $\lambda$,
we define $\Vh \sqsubseteq \Vt $ if $v_{h,i} \sqsubseteq v_{t,i}$ for all $\rangeco{i}{0}{\lambda}$.
For instance, our previous example trace $\boldsymbol{v}'$ satisfies $\boldsymbol{v}' \sqsubseteq
\left\{ x \mapsto 4          \right\} \cdot
\left\{ x \mapsto 1 \right\} \cdot
\left\{ x \mapsto 5          \right\}$
because both traces coincide in the variables defined in $\boldsymbol{v}'$, but the new trace assigns $x \mapsto 1$ at the second state, while $\boldsymbol{v}'$ left it undefined.
As before, we say $\Vh = \Vt$ if $v_{h,i} = v_{t,i}$ for all $\rangeco{i}{0}{\lambda}$ and
$\Vh \sqsubset \Vt$ if $\Vh \sqsubseteq \Vt$ and $\Vh \not = \Vt$.
All this allows us to define an \emph{\HTC\ trace} of length $\lambda$ as
a pair $\tuple{\Vh,\Vt}$ of traces $\Vh$ and $\Vt$ of length $\lambda$ such that $\Vh \sqsubseteq \Vt$.

\HTC\ evaluates constraint atoms using denotations,
a mechanism that abstracts from the syntax and semantics of expressions originating from external theories.
In \HTC, a valuation can be seen as a solution of a constraint.
For instance the denotation $\den{x+y \le 5}$ contains all valuations $v$ such that $v(x) + v(y) \le 5$.
Also, the satisfaction of a constraint atom corresponds in \HTC\ to checking whether a valuation belongs to a denotation or not.
In \THTC, however, time-varying variables are handled.
For instance the constraint $x + \next^5 x \le 5$ is satisfied at time point $t$ when
the addition of the values of $x$ at $t$ and $t+2$ is smaller or equal to $5$.
In the temporal case, evaluating a temporal constraint atom requires the use of a finite number of valuations,
for which we would need a more complex denotational approach.
Therefore, we adopt a simpler approach,
associating each constraint atom with a relation, where each tuple signifies a valid value assignment,
as is standard in CSPs:
Given a temporal constraint atom $c$ of arity $n$,
we define its \emph{solution relation} as $\tbl{c}\subseteq\mathcal{D}_{\undefined}^n$.
Note that this definition tolerates undefined variables.
Given a constraint atom $c$ of arity $n$,
$\tbl{c}$ is said to be \emph{strict} whenever $\tbl{c} \subseteq \mathcal{D}^n$,
i.e., if it leaves no variables undefined.
Otherwise, we say that $\tbl{c}$ is \emph{non-strict}.

For instance, to satisfy the atom $x + \next^5 x \le 5$ neither term $x$ nor $\next^5 x$ can be undefined.
A natural choice for its (strict) solution relation is thus
\begin{align*}
  \tbl{x + \next^5 x \le 5} & = \{ (a,b) \in \mathbb{Z}^2 \mid a+b\le 5\}.
\end{align*}
As an example of an atom with a non-strict solution, we could define an atom such as $\mathit{some\_zero}( \previous x, x, \next x)$ requiring that $x$ equals zero at the previous, current, or next state,
but allowing for $x$ to be undefined in some of those states:
\begin{align*}
  \tbl{\mathit{some\_zero}( \previous x, x, \next x )} & =
  \{ (a,b,c) \in (\mathbb{Z} \cup \{\undefined\})^3 \mid  a=0 \text{ or } b=0 \text{ or } c=0\}.
\end{align*}

Given a trace $\boldsymbol{v}=(v_i)_{\rangeo{i}{0}{\lambda}}$ of length $\lambda$,
the \emph{value} of a term $\next^ox$ at time point $\rangeo{i}{0}{\lambda}$ is defined as
\begin{align*}
  v_i (\next^ox ) &\eqdef
                      \begin{cases}
                        v_{i+o}(x)  & \text {if } \rangeo{i+o}{0}{\lambda} \\
                        \undefined & \text {otherwise}.
                      \end{cases}
\end{align*}

Given an \HTC\ trace $\M=\tuple{\Vh,\Vt}$ and a time point $\rangeco{i}{0}{\lambda}$,
we define the satisfaction of a formula in \THTC\ as follows:
\begin{enumerate}
\item\label{sat:atoms}
  $\M,i\models c(\next^{o_1}x_1,\dots,\next^{o_n}x_n)$ iff for all $w \in \{ h,t\}$\\\hspace*{50pt}
  $(v_{w,i}(\next^{o_1} x_1),\dots,v_{w,i}(\next^{o_n} x_n))\in \tbl{c(\next^{o_1}x_1,\dots,\next^{o_n}x_n)}$
\item $\M,i\models \varphi \wedge \psi$ iff $\M,i\models \varphi$ and $\M,i\models \psi$
\item $\M,i\models \varphi \vee \psi$ iff $\M,i\models \varphi$ or $\M,i\models \psi$
\item $\M,i\models\varphi \to \psi$ iff $\tuple{\Vw,\Vt},i\not\models\varphi$ or $\tuple{\Vw,\Vt},i\models\psi$ for all $w \in \{ h,t\}$
\item $\M,i\models\next\varphi$  iff $i<\lambda-1$ and $\M , i +1 \models \varphi$
\item $\M,i\models\varphi \until \psi$  iff there exists $\rangeo{k}{i}{\lambda}$ s.t.\ $\M , k \models \psi$ and $\M , j \models \varphi$ for all $i \le j < k$
\item $\M,i\models\varphi \release \psi$ iff for all $\rangeo{k}{i}{\lambda}$, either  $\textbf{M} , k \models \psi$ or $\M , j \models \varphi$ for some $i \le j < k$
\item $\M,i\models\previous \varphi$  iff $i>0$ and $\M , i -1 \models \varphi$
\item $\M,i\models\varphi \since \psi$  iff there exists $\rangec{k}{0}{i}$ s.t.\ $\M , k \models \psi$ and $\M , j \models \varphi$ for all $k < j \le i$
\item $\M,i\models\varphi \trigger \psi$ iff for all $\rangec{k}{0}{i}$, either  $\textbf{M} , k \models \psi$ or $\M , j \models \varphi$ for some $k < j \le i$
\end{enumerate}
In fact, given a constraint atom $c$ with a strict relation $\tbl{c}$,
the satisfaction of $c$ depends solely on the ``here'' component of the trace~\citep{cakaossc16a}.
In this case, Condition~\ref{sat:atoms} can be replaced by:
\begin{enumerate}
\item[\ref{sat:atoms}'.] $\M,i\models c(\next^{o_1}x_1,\dots,\next^{o_n}x_n)$ iff\\\hspace*{50pt}
  $(v_{h,i}(\next^{o_1} x_1), \dots,  \allowbreak v_{h,i}(\next^{o_n} x_n))\in  \tbl{c(\next^{o_1}x_1,\dots,\next^{o_n}x_n)}$.
\end{enumerate}
\begin{proposition}\label{prop:derived:op}
  Given an \HTC\ trace $\M=\tuple{\Vh,\Vt}$ and a time point $ \rangeco{i}{0}{\lambda}$,
  we have the following satisfaction relations for the derived operators:
  \begin{enumerate}
  \item $\M,i\models \alwaysP \varphi$ iff for all $\rangeoc{j}{0}{i}$, $\M, j \models \varphi$
  \item $\M,i\models \eventuallyP \varphi$ iff there exists $\rangeoc{j}{0}{i}$ such that $\M, j \models \varphi$
  \item $\M,i\models \initially$ iff $i=0$
  \item $\M,i\models \wprevious \varphi$ iff either $i = 0$ or $\M, i-1 \models \varphi$
  \item $\M,i\models \alwaysF \varphi$ iff for all $\rangeco{j}{i}{\lambda}$, $\M, j \models \varphi$
  \item $\M,i\models \eventuallyF \varphi$ iff there exists $\rangeco{j}{i}{\lambda}$ such that $\M, j \models \varphi$
  \item $\M,i\models \finally$ iff $i = \lambda -1$
  \item $\M,i\models \wnext \varphi$ iff either $i = \lambda-1$ or $\M, i+1\models \varphi$
  \end{enumerate}
\end{proposition}

For illustration,
consider the short \HTC\ trace \tuple{\Vh,\Vt} where
\begin{align*}
  \Vh &= \left\{ x\mapsto 4, y\mapsto \undefined \right\}\cdot \left\{ x \mapsto 5, y \mapsto \undefined\right\} \cdot \left\{ x \mapsto \undefined , y \mapsto \undefined \right\} \cdot\left\{ x \mapsto 5, y \mapsto 6\right\}  \\
  \Vt &= \left\{ x\mapsto 4, y\mapsto 6 \right\}\cdot  \left\{ x \mapsto 5, y \mapsto \undefined\right\} \cdot \left\{ x \mapsto 4, y \mapsto 5\right\} \cdot \left\{ x \mapsto 5, y \mapsto 6         \right\}
\end{align*}
As a first example, consider
the formula $(x=4) \wedge (\next x <  \next^3y)$ whose atoms have the solution relations
$\tbl{x=4} = \{(4) \}$ and $\tbl{\next x < \next^3 y} = \{ (a,b) \in \mathbb{Z}^2 \mid a < b\}$.
Note that each solution relation above contains only the assignments that satisfy the constraint and they do not depend on the ``next'' $\next^k$ prefix that qualifies each variable inside a constraint.
Broadly speaking, a solution for $(\next x <  \next^3y)$ consists of two assignments,
one for $x$ and one for $y$ that satisfy the constraint.
The effect of using the term $\next^3 y$ is that the value for $y$ is given by the valuation placed three states ahead in the trace,
i.e., the valuation of $\next^3 y$ depends on the sequence of valuations in the trace.
Moreover, both $\tbl{x=4}$ and $\tbl{\next x < \next^3 y}$ are strict,
so we can use Condition~\ref{sat:atoms}' to evaluate the constraint atoms.
We have therefore $\tuple{\Vh, \Vt}, 0 \models (x=4)$.
Also, we have $\tuple{\Vh,\Vt},0 \models (\next x < \next^3 y)$ since $v_{h,3}(y) = 6$, $v_{h,1}(x) = 5$ and $(5,6) \in \tbl{\next x < \next^3 y}$.
Therefore, $\tuple{\Vh,\Vt}, 0 \models (x=4)\wedge (\next x < y)$.

Next,
consider the formula
$(\previous x < 7)$ along with\footnote{We sometimes identify 1-tuples like $\tuple{a}$ with the  element itself $a$.} $\tbl{\previous x < 7} = \{ a \in \mathbb{Z} \mid a < 7\}$.
Since there is no time point before the initial state, we get $v_{h,0}(\previous x) = \undefined$.
Hence, $\tuple{\Vh, \Vt}, 0 \not \models \previous x<7$
because $\undefined\notin\tbl{\previous x < 7}$.
However, if the constraint is evaluated at time point $1$,
we get $\tuple{\Vh, \Vt}, 1 \models \previous x<7$ as $v_{h,1}(\previous x)=4$ and $4\in\tbl{\previous x < 7}$.

Finally, consider the equation $y = y$ with $\tbl{y = y} = \mathbb{Z}$.
Clearly, we have $\tuple{\Vh, \Vt}, 3 \models y=y$ since $v_{h,3}(y)=6$ and $6 \in \tbl{y = y}$.
However, this is not the case at time point $0$ because $v_{h,0}(y)=\undefined$ and $\undefined\notin\tbl{y = y}$.

Next, we show that \THTC\ satisfies the characteristic properties of \HT-based logics.
\begin{proposition}[Persistence, Negation]\label{prop:persistence}\label{cor:negation}
  For all \HTC\ traces $\tuple{\Vh,\Vt}$ of length $\lambda$,
  all formulas $\varphi$, and all $\rangeo{i}{0}{\lambda}$ we have:
  \begin{enumerate}
  \item $\tuple{\Vh,\Vt}, i \models \varphi$ implies $\tuple{\Vt,\Vt},i \models \varphi$.
  \item $\tuple{\Vh,\Vt}, i \models \neg \varphi$ iff $\tuple{\Vt,\Vt}, i \not\models \varphi$.
  \end{enumerate}
\end{proposition}

Finally,
temporal equilibrium models are defined in the traditional way. \begin{definition}\label{def:teqm}
  An \HTC\ trace $\tuple{\Vt,\Vt}$ of length $\lambda$ is a temporal equilibrium model of a formula $\varphi$ if
  \begin{enumerate}
  \item $\tuple{\Vt,\Vt}, 0 \models \varphi$ and
  \item there is no \HTC\ trace $\tuple{\Vh,\Vt}$ such that $\Vh\sqsubset\Vt$ and
    $\tuple{\Vh,\Vt}, 0 \models \varphi$.
  \end{enumerate}
\end{definition}
Given that equilibrium models are the semantic counterpart of stable models in ASP,
we may thus refer to \Vt\ as a stable model of $\varphi$ in \THTC.

 \subsection{\THTC\ as a Conservative Extension of \HTC}\label{sec:htc}

\THTC\ can be viewed as a temporal extension of \HTC~\citep{cakaossc16a}.
We support this in what follows by showing that both coincide on non-temporal formulas.

To achieve this, we need to establish a correspondence between the semantic concepts of denotations and solution relations.
Given a signature~$\tuple{\mathcal{X},\mathcal{D},\mathcal{A}}$,
let $\mathcal{V}$ be the set of all partial valuations from $\mathcal{X}$ to $\mathcal{D}_{\undefined}$.
A \emph{denotation} is a function $\den{\cdot}: \mathcal{A} \to 2^{\mathcal{V}}$
assigning a set of partial valuations to each atom.
In \HTC, one considers denotations with a closed range:
if a partial valuation $v$ satisfies an atom $c$, $v\in\den{c}$,
then any partial valuation $v'$ that extends $v$, $v \sqsubseteq v'$, also satisfies $c$.
\begin{align*}
  \den{x = y}	                &\eqdef \{ v \in \mathcal{V} \mid v(x)=v(y)\not=\undefined \}\\
\den{\mathit{some\_zero}(X)}	&\eqdef \{ v\in \mathcal{V} \mid v(x) = 0 \text{ for some }x \in X \}.
\end{align*}

For a (non-temporal) constraint $c(x_1,\cdots,x_n)$ in $\mathcal{A}$,
only referring to variable values in the current state,
we define
\begin{align*}
  \tbl{c(x_1,\cdots,x_n)} & = \{ (v(x_1), \cdots, v(x_n)) \mid v \in \mathcal{V}\}\\
  \den{c(x_1,\cdots,x_n)} & = \{ v\in \mathcal{V}         \mid (v(x_1), \cdots, v(x_n))\in \tbl{c(x_1,\cdots,x_n)} \}.
\end{align*}
Given this correspondence, we can show the following result,
where \HTCSAT\ and \THTCSAT\ denote the satisfaction relation of \HTC~\citep{cakaossc16a} and \THTC\ (as given above).
\begin{proposition}\label{prop:htc2thtc}
  For all formulas $\varphi$ containing only Boolean connectives and
  all \HTC\ interpretations $\tuple{v_h, v_t}$, we have
  $\tuple{v_h, v_t} \HTCSAT \varphi$ iff  $(\tuple{v_h, v_t}),0 \THTCSAT \varphi$.
\end{proposition}
Since \HTC\ interpretations are simply \HTC\ traces of length 1,
it follows that \HTC\ and \THTC\ coincide when restricted to the language of \HTC.

 \subsection{\THTC\ as a Conservative Extension of \THTf}\label{sec:tht2thtc}

\THTC\ extends \THTf~\citep{agcadipescscvi20a} by including constraints.
Similar to the previous section,
we confirm this relationship by showing that they coincide on temporal formulas without non-Boolean constraints.
To proceed, we first need to introduce basic concepts of \THTf.

\THTf\ employs the same syntax for temporal formulas as \THTC,
with the key distinction that constraint atoms are replaced by Boolean atoms in \var.
The semantics of \THTf\ relies on \HT\ traces.
An underlying (Boolean) trace of length $\lambda$ is a sequence $(T_i)_{\rangeco{i}{0}{\lambda}}$
where $T_i\subseteq \var$ for all $\rangeo{i}{0}{\lambda}$.
An \HT\ trace $\tuple{\Htrace,\Ttrace}$ is a pair of traces
$\Htrace = (H_i)_{\rangeco{i}{0}{\lambda}}$ and $\Ttrace = (T_i)_{\rangeco{i}{0}{\lambda}}$
satisfying $H_i\subseteq T_i$ for all $\rangeco{i}{0}{\lambda}$.
An \HT\ trace $\tuple{\Htrace,\Ttrace}$ satisfies a Boolean atom $p \in \var$ at $\rangeco{i}{0}{\lambda}$ if $p \in H_i$.
The satisfaction conditions for more complex formulas are identical to those of \THTC.

To encode \THTf\ within \THTC,
we consider the signature $\tuple{\var ,\{\mathtt{t}\}, \lbrace p = \mathtt{t}\mid p \in \var \rbrace }$,
where $\mathtt{t}$ is considered as true.\footnote{In \HTC, Boolean variables are already captured by truth values $\mathtt{t}$ and \undefined\ (rather than $\mathtt{f}$[alse])~\citep{cakaossc16a}.}
Accordingly, we set $\tbl{p=\mathtt{t}} = \{ \mathtt{t}\}$ for all $p\in\var$.
Then,
we define the bijective mapping $\delta$ from $\HT{}$ traces to $\HTC$ traces as follows.\footnote{We demonstrate that $\delta$ is a bijective function in Proposition~\ref{prop:tht2thtc:bijection} in the supplementary material.}
Given an \HT\ trace $\tuple{\Htrace,\Ttrace} = \tuple{(H_i)_{\rangeco{i}{0}{\lambda}},(T_i)_{\rangeco{i}{0}{\lambda}}}$ of length $\lambda$,
we define the corresponding \HTC\ trace $\delta(\tuple{\Htrace,\Ttrace}) = \tuple{\Vh,\Vt}$ of length $\lambda$
where for all $\rangeco{i}{0}{\lambda}$ and for all $p \in \var$,
we have
\begin{align*}
  v_{h,i}(p) &=
               \begin{cases}
                 \mathtt{t} & \text{ if } p \in H_i\\
                 \undefined & \text{ otherwise}
               \end{cases}
  & v_{t,i} (p) &=
                  \begin{cases}
                    \mathtt{t} & \text{ if } p \in T_i\\
                    \undefined & \text{ otherwise}
                  \end{cases}
\end{align*}
Then, the following proposition can be easily proved via structural induction.
\begin{proposition}\label{prop:tht2thtc}
  For any \HT\ trace $\tuple{\Htrace,\Ttrace}$,
  we have for all temporal formulas $\varphi$ and all $\rangeco{i}{0}{\lambda}$ that
  $\tuple{\Htrace,\Ttrace},i \models \varphi$ iff $\delta(\tuple{\Htrace,\Ttrace}),i \models \varphi'$,
  where $\varphi'$ is obtained from $\varphi$ by replacing every atom $p$ by the constraint $p=\mathtt{t}$.
\end{proposition}

Proposition~\ref{prop:tht2thtc} can be extended to the case of equilibrium models.
\begin{proposition}\label{prop:tht2thtc:equilibrium}
For any $\HT$ trace $\tuple{\Htrace,\Ttrace}$ and for any formula $\varphi$, $\tuple{\Vh,\Vt}$ is a temporal equilibrium model of $\varphi$ iff $\delta(\tuple{\Htrace,\Ttrace})$ is a temporal equilibrium model (under \THTC{} semantics) of $\varphi$.
\end{proposition}

 \section{From \THTC\ to Quantified \HT\ with Evaluable Functions}\label{sec:firstorder}

Kamp's translation~\citeyearpar{kamp68a} is a cornerstone result that provides a fundamental link between temporal and classical logic,
offering deep theoretical insights and practical implications for the study and application of temporal reasoning.

As a step towards a similar result in our non-classical context,
we define a translation from \THTC\ into Quantified\footnote{``Quantified'' is used synonymously with ``First-order''.} \HT\
with decidable equality, evaluable functions~\citep{cabalar11a}, and an order relation. We consider a first-order language with signature $\tuple{\mathcal{C},\mathcal{F},\mathcal{P}}$,
where $\mathcal{C}$ and $\mathcal{F}$ are disjoint sets of uninterpreted and evaluable function names, respectively,
and $\mathcal{P}$ is a set of predicate names including equality ($=$) and a strict order relation $<$
(the non-strict version $\leq$ is defined as usual).
We assume that $\undefined\notin\mathcal{C}\cup\mathcal{F}$ and $+\in\mathcal{F}$.
First-order formulas are built in the usual way, defining $\top$, $\neg$, $\leftrightarrow$ as above.

The difference of our translation to Kamp's original one lies in the fact that
our predicate names are not monadic and that
we allow (partial) function symbols for simulating the valuation of a variable along time.

Given a temporal formula $\varphi$ in \THTC,
we define the first-order formula $\ST{t}{\varphi}$
with the only free variable $t$ as follows.
For constraint atoms, our translation is defined as
\begin{align*}
  \ST{t}{c(\next^{o_1}x_1,\dots,\next^{o_n}x_n)}  = p_c(f_{x_1}(t+o_1),\dots,f_{x_n}(t+o_n)).
\end{align*}
where $p_c$ is a predicate that simulates the behavior of $c$ and
each $f_{x_i}\in\mathcal{F}$ is a partial evaluable function associated with the variable $x_i$ for $1\leq i\leq n$.
The value of a variable $x$ at a time point $t$ is given by $f_x(t)$.
To capture the temporal offset $i$ in $\next^ix$, we shift the argument of $f_x$, resulting in $f_x(t+i)$.
Note that $f_x$ is partial because its argument might extend beyond the defined time points in a trace.

The rest of the translation follows the one by~\cite{kamp68a}:
\begin{align*}
  \ST{t}{\bot}                     &= \bot \\
  \ST{t}{\varphi \wedge \psi}      &= \ST{t}{\varphi}\wedge\ST{t}{\psi} \\
  \ST{t}{\varphi \vee \psi}        &= \ST{t}{\varphi}\vee\ST{t}{\psi} \\
  \ST{t}{\varphi \to \psi}         &= \ST{t}{\varphi}\to\ST{t}{\psi} \\
  \ST{t}{\next \varphi}            &= \exists t' (t'=t+1) \wedge \ST{t}{\varphi}\\
  \ST{t}{\varphi \until \psi}      &= \exists t' (t\le t' \wedge \ST{t'}{\psi} \wedge (\forall t''( (t\le t'' \wedge t'' < t')\to \ST{t''}{\varphi} )))\\
  \ST{t}{\varphi \release \psi}    &= \forall t' (t\le t' \to   (\ST{t'}{\psi} \vee   (\exists t''(  t\le t'' \wedge t'' < t'\wedge \ST{t''}{\varphi} ))))\\
  \ST{t}{\previous \varphi}        &= \exists t' (t=t'+1) \wedge \ST{t}{\varphi}\\
  \ST{t}{\varphi \since \psi}      &= \exists t' (t'\le t \wedge \ST{t'}{\psi} \wedge (\forall t''( (t' < t'' \wedge t'' \le t )\to \ST{t''}{\varphi} )))\\
  \ST{t}{\varphi \trigger \psi}    &= \forall t' (t'\le t \to   (\ST{t'}{\psi} \vee   (\exists t''(  t' < t'' \wedge t'' \le t \wedge \ST{t''}{\varphi} ))))
\end{align*}
The translation of derived operators can be done unfolding their definitions or using the shorter, equivalent formulas:
\begin{align*}
  \ST{t}{\initially } &= \neg \exists t' (t'<t) & \ST{t}{\finally} &= \neg \exists t' ( t < t')\\
  \ST{t}{\wprevious \psi} &= \forall t' ( t=t'+1 \to \ST{t'}{\psi}) & \ST{t}{\wnext \psi} &= \forall t' ( t'=t+1 \to \ST{t'}{\psi})\\
  \ST{t}{\eventuallyF \psi} &= \exists t' ( t \le t'\wedge \ST{t'}{\psi}) & \ST{t}{\eventuallyP \psi} &= \exists t' ( t'\le t \wedge \ST{t'}{\psi})\\
  \ST{t}{\alwaysF \psi}     &= \forall t' ( t \le t'\to    \ST{t'}{\psi}) & \ST{t}{\alwaysP \psi}     &= \forall t' ( t'\le t \to    \ST{t'}{\psi})
\end{align*}

As an example, let us consider the formula $\alwaysF\eventuallyP( \next^2 x= x)$ and the free variable $t$.
The result $\ST{t}{\alwaysF\eventuallyP\left( \next^2 x= x \right)}$ of translating this formula is as follows:
\begin{align*}
  \ST{t}{\alwaysF\eventuallyP( \next^2 x= x  )} &= \forall t' (t \le t' \to  \ST{t'}{\eventuallyP ( \next^2 x= x    )}) \\
                                                &= \forall t' (t \le t' \to ( \exists t'' ( t''\le t' \wedge \ST{t''}{\next^2 x= x} )    ))\\
                                                &= \forall t' (t \le t' \to ( \exists t'' ( t''\le t' \wedge (f_x(t''+2)= f_x(t'')) )    ))
\end{align*}
In our non-classical framework,
the interpretation of first-order formulas requires Quantified \HT\ Logic with evaluable functions~\citep{cabalar11a}.
In preparation for establishing the correctness of our translation in Proposition~\ref{lem:kamp},
we proceed to introduce the semantics of this logic.
Given signature $\tuple{\mathcal{C},\mathcal{F}, \mathcal{P}}$, we define:
\begin{itemize}
\item $\TERMS{\mathcal{C}}$ as the set of all ground terms over $\mathcal{C}$,
\item $\TERMS{\mathcal{C} \cup \mathcal{F}}$ as the set of all ground terms over $\mathcal{C}$ and $\mathcal{F}$,
\item $\ATOMS{\mathcal{C},\mathcal{P}}$ as the set of all ground atoms over $\mathcal{C}$ and $\mathcal{P}$.
\end{itemize}

A \emph{state} over $\tuple{\mathcal{C},\mathcal{F},\mathcal{P}}$ is a pair $(\sigma,A)$,
where
$A \subseteq \ATOMS{\mathcal{C}, \mathcal{P}}$ and
$\sigma: \TERMS{\mathcal{C} \cup \mathcal{F}} \rightarrow \TERMS{\mathcal{C}} \cup \{ \undefined \rbrace$ is a
function such that
\begin{enumerate}
\item $\sigma(s) = s$ \qquad\qquad\ for all $s \in \TERMS{\mathcal{C}}$
\item $\sigma(f(s_1,\dots,s_n))=
  \begin{cases}
    \undefined                               & \text{ if } \sigma(s_i) =  \undefined \text{ for some } i\in\{1,\dots,n\} \\
    \sigma(f(\sigma(s_1),\dots,\sigma(s_n))) & \text{ otherwise}
  \end{cases}$
\end{enumerate}

Given two states $S=(\sigma,A)$ and $S'=(\sigma',A')$, we write $S \peq S'$ when both
\begin{enumerate}
\item $A\subseteq A'$ and
\item $\sigma(s) = \sigma'(s)$ or $\sigma(s) = \undefined$ for all $s \in \TERMS{\mathcal{C}\cup \mathcal{F}}$.
\end{enumerate}
We also write $S \pt S'$ when $S\peq S'$ but $S\not = S'$.

Following~\cite{cabalar11a}, a \QHT\ interpretation is a structure $\tuple{S_h,S_t}$, where
$S_h = (\sigma_h,I_h)$ and $S_t = (\sigma_t, I_t)$ such that $S_h \peq S_t$.
Given a first-order formula $\varphi$ and a \QHT\ interpretation $\mathcal{M}=\tuple{S_h,S_t}$,
we define the satisfaction relation as follows:
\begin{enumerate}
\item $\mathcal{M} \models \top$ and $\mathcal{M} \not \models \bot$
\item $\mathcal{M} \models p(s_1,\dots,s_n)$ if $p(\sigma_h(s_1),\dots,\sigma_h(s_n)) \in I_h$
\item $\mathcal{M} \models s=s'$ if $ \sigma_h(s) = \sigma_h(s')\not=\undefined$
\item $\mathcal{M} \models s < s'$ if $\undefined \not = \sigma_h(s) < \sigma_h(s')\not=\undefined$
\item $\mathcal{M} \models \varphi \wedge \psi $ if $\mathcal{M} \models\varphi \text{ and } \mathcal{M} \models\psi$
\item $\mathcal{M} \models \varphi \vee \psi $ if $\mathcal{M}  \models\varphi \text{ or } \mathcal{M}  \models\psi$
\item $\mathcal{M} \models \varphi \to \psi $ if $\mathcal{M}' \not \models \varphi$ or $\mathcal{M}' \models \psi$ for all $\mathcal{M}'\in \{\mathcal{M}, (S_t,S_t)\}$
\item $\mathcal{M} \models \exists\, x\ \varphi(x) $ if $\mathcal{M} \models \varphi(c)$ for some  $c \in \TERMS{\mathcal{C}}$
\item $\mathcal{M} \models \forall\, x\ \varphi(x) $ if $\mathcal{M} \models \varphi(c)$ for all $c\in \TERMS{\mathcal{C}}$.
\end{enumerate}
A \QHT\ interpretation $\mathcal{M}$ is a \QHT\ model of a formula $\varphi$ when $\mathcal{M} \models \varphi$.
A \QHT\ model $\tuple{S_t,S_t}$ is an equilibrium model for a formula $\varphi$,
if there is no state $S_h$ such that $S_h \pt S_t$ and $\tuple{S_h, S_t}$ is a \QHT\ model of $\varphi$.

We are now ready to define the model correspondence between \HTC\ traces and \QHT\ interpretations.
An \HTC\ trace $\M=\tuple{\Vh,\Vt}$ of length $\lambda$ for signature~$\tuple{\mathcal{X},\mathcal{D},\mathcal{A}}$ corresponds to a \QHT\ interpretation $\mathcal{M}=\tuple{(\sigma_h,I_h),(\sigma_t,I_t)}$
iff
for each $\rangeco{i}{0}{\lambda}$ and for each variable $x \in \mathcal{X}$,
we get $\sigma_{h} (f_{x}(i)) = v_{h,i}(x)$, $\sigma_{t} (f_{x}(i))  = v_{t,i}(x)$,
\begin{align*}
  I_h & = \{p_c(a_1,\dots, a_n) \mid  c(\next^{o_1}x_1,\dots,\next^{o_n}x_n) \in \mathcal{A}, \rangeco{i}{0}{\lambda} \hbox{ and }\\
      & \qquad c(v_{h,i+o_1}(x_1),\dots, v_{h,i+o_n}(x_n))=c(a_1,\dots, a_n) \in \tbl{c(\next^{o_1}x_1,\dots,\next^{o_n}x_n)}\} \hbox{ and }\\
  I_t & = \{p_c(a_1,\dots, a_n) \mid  c(\next^{o_1}x_1,\dots,\next^{o_n}x_n) \in \mathcal{A}, \rangeco{i}{0}{\lambda} \hbox{ and }\\
      & \qquad c(v_{t,i+o_1}(x_1),\dots, v_{t,i+o_n}(x_n))=c(a_1,\dots, a_n) \in \tbl{c(\next^{o_1}x_1,\dots,\next^{o_n}x_n)}\},
\end{align*}
where $\tbl{c(\next^{o_1}x_1,\dots, \next^{o_n}x_n)}$ is strict.

This model correspondence allows us to translate \THTC\ into \QHT.
We let $\varphi[t/i]$ stand for the result of replacing each occurrence of variable $t$ in $\varphi$ by $i$.
\begin{proposition}\label{lem:kamp}
  Let $\M = \tuple{\Vh,\Vt}$ be a \HTC\ trace and let $\mathcal{M} = \tuple{(\sigma_h,I_h),(\sigma_t,I_t)}$ be its corresponding \QHT\ interpretation.
  For all $\rangeco{i}{0}{\lambda}$ and every temporal formula $\varphi$,
  we have $\M, i \models \varphi$ iff $\mathcal{M} \models \ST{t}{\varphi}[t/i]$.
\end{proposition}
In other words, we can consider \THTC\ as a subclass of theories in the fragment of \QHT{} with partial evaluable functions.
We show next that this correspondence is still  valid when considering equilibrium models.
\begin{proposition}\label{lem:kamp:equilibrium}
  Let $\M = \tuple{\Vh,\Vt}$ be a \HTC\ trace and let $\mathcal{M} = \tuple{(\sigma_h,I_h),(\sigma_t,I_t)}$ be its corresponding \QHT\ interpretation.
  For every temporal formula $\varphi$,
  $\M$ is an equilibrium model of $\varphi$ iff $\mathcal{M}$ is an equilibrium model of $\ST{t}{\varphi}[t/0]$.
\end{proposition}

An important consequence of our translation is that it opens the possibility of applying first-order theorem provers for solving inference problems for \THTC.
For instance, one salient \THTC\ inference problem is deciding \THTC-equivalence of two temporal theories with constraints, $\Gamma_1$ and $\Gamma_2$,
since it is a sufficient condition for their \emph{strong equivalence}~\citep{lipeva07a},
something that guarantees a safe replacement of $\Gamma_1$ by $\Gamma_2$ (or vice versa) in any arbitrary context.
The translation of the \QHT-formula obtained from $\ST{t}{\cdot}$ into classical First-Order Logic can be done in two steps:
first removing the partial functions in favor of predicates~\citep{cabalar11a}, and
second passing from quantified \HT\ (without partial functions) into first-order classical logic~\citep{pearce06a}.

 \section{Logic programs with Temporal Linear Constraints}\label{sec:linear}

In this section, we consider the fragment of \THTC\ interpreted over linear constraints whose solution table is strict.
We define a \textit{linear term} as an expression of the form
\begin{align}\label{def:linear:term}
  d_1 \cdot \next^{o_1} x_1 + \dots + d_n \cdot \next^{o_n} x_n
\end{align}
where $d_i \in \mathbb{Z}$, $x_i \in \mathcal{X}$, and $o_i\in \mathbb{Z}$ for $0\leq i\leq n$.
Multiplication and addition are denoted by `$\cdot$' and `$+$', respectively.
An implicit multiplicative factor of 1 is assumed for standalone numbers $d_i$.
Negative constants are represented using `$-$',
and the `$\cdot$' symbol may be omitted when contextually clear.
Given a linear term $\alpha$ as in~\eqref{def:linear:term},
we define
\(
\vars(\alpha)=\{\next^{o_i} x_i\mid 0\leq i\leq n\}
\).
For instance, $\vars(2 \cdot \next^{-2} x) = \{  \next^{-2} x\} = \{ \previous \previous x\}$.

A \textit{linear temporal constraint atom} is a temporal constraint atom of the form $\alpha\le\beta$,
where $\alpha$ and $\beta$ are linear terms.
We use the abbreviation
$\alpha = \beta$,
$\alpha < \beta$, and
$\alpha \neq \beta$
for
$(\alpha \le \beta) \wedge (\beta \le \alpha)$,
$(\alpha \le \beta) \wedge \neg (\beta \le \alpha)$, and
$(\alpha < \beta) \vee (\beta < \alpha)$, respectively.
Similarly, $\vars(\alpha \otimes \beta)= \vars(\alpha) \cup \vars(\beta)$, for $\otimes \in \lbrace \le, < , =, \neq\rbrace$.
Given our focus on strict solution tables, we have
\(
\tbl{\alpha\le\beta} \subseteq \mathcal{D}^n
\)
for all $\alpha\le\beta$,
where $n$ is the number of terms in $\alpha$ and $\beta$.
More generally, given a constraint atom $c$ with a strict relation,
we define its \emph{complement} constraint 
as $\tbl{\overline{c}} = \mathcal{D}^n \setminus\tbl{c}$ (which is also strict).
The interaction of complementary constraints with logical negation is given in the next proposition.
\begin{proposition}\label{prop:complement}
  For any \HTC\ trace $\M=\tuple{\Vh,\Vt}$ of length $\lambda$ and
  any pair $c$, $\overline{c}$ of complementary constraints, with both $\tbl{c}$ and $\tbl{\overline{c}}$ being strict,
  we have that
  $\M, i \models c$ implies $\M, i \models \neg \overline{c}$
  for all $\rangeco{i}{0}{\lambda}$.
\end{proposition}
We thus obtain that $\alpha = \beta$ implies $\neg (\alpha \neq \beta)$ but not necessarily vice versa.

For illustration, consider the atoms $\next x = x$ and $\next x \not= x$, where $x$ takes its values in $\mathbb{N}$.
We define the solution relation for each constraint as
\begin{align*}
  \tbl{\next x = x}      & =  \{ (a_1,a_2) \in \mathbb{N}^2 \mid a_1 = a_2 \not = \undefined \}\\
  \tbl{\next x \not = x} & =  \{ (a_1,a_2) \in \mathbb{N}^2 \mid a_1 \not= a_2 \text{ and } a_1 \not = \undefined \text{ and } a_2 \not= \undefined\}.
\end{align*}
We have that $\tbl{\next x \not = x} = \mathbb{N}^2\setminus \tbl{\next x = x}$, and therefore $\tbl{\next x \not = x} = \tbl{\overline{\next x  = x}}$.
By Proposition~\ref{prop:complement}, $\next x = x$ implies $\neg (\next x \not= x)$.
Conversely, let $\M=\tuple{\Vh,\Vt}$ be an \HTC\ trace of length $2$ defined as $v_{w,i}(x) = \undefined$ for $w \in \{ h,t\}$ and $i\in \{ 0,1\}$.
Then, we have that $\M, 0 \models \neg ( \next x \not = x )$ but $\M, 0 \not \models \next x = x$.

When \HTC\ is extended with linear constraints,
\cite{cakaossc16a} introduce an interesting feature known as \emph{assignments}.
This assignment operator is a specialized form of equality
that allows us to check if a variable has a defined value.
To illustrate this, consider the two following formulas:
\footnote{Recall that atom $p=\mathtt{t}$ is the constraint representation of variable $p$ (see Section~\ref{sec:tht2thtc}).}
\begin{align}\label{eq:formulas:undefined}
  (x = y)
  \text{ and }
  (p=\mathtt{t} \to y = 10)
\end{align}
The equality $x = y$ implies that any assignment to $x$ must also be an assignment to $y$, and vice versa.
Given that $p=\mathtt{t}$ is not derivable, $y$ appears to be undefined in this context.
However, due to the strict nature of $\tbl{x = y}$,
this allows for both $x$ and $y$ to take on any value, even if $y$ seems undefined.
To ensure that $x$ is only assigned a value when $y$ is defined, we can replace $x = y$ in~\eqref{eq:formulas:undefined} with:
\begin{align}\label{eq:formulas:defined}
  y \leq y & \to x = y
\end{align}
Because $y\leq y$ cannot be established (due to the underivability of $p=\mathtt{t}$ making $y$ undefined),
the resulting \HTC\ model has $x$, $y$ and $p$ all undefined.
Indeed, within \HTC, the modified formula in~\eqref{eq:formulas:defined} effectively represents the assignment $x := y$.

The remainder of this section focuses on extending this assignment mechanism of  \HTC\ to our temporal setting in \THTC.
For any linear expression $\alpha$, we define
\begin{align*}
  \df(\alpha) & \eqdef\textstyle \bigwedge_{\next^l x \in \vars(\alpha)} \next^l x\le \next^l x.
\end{align*}
The satisfiability of $\df(\alpha)$ necessitates that all its variables are defined, as shown next.
\begin{proposition}\label{prop:dfs}
  For all \HTC\ traces $\M=\tuple{(v_{h,i})_{\rangeco{i}{0}{\lambda}},(v_{t,i})_{\rangeco{i}{0}{\lambda}}}$ and
  for all linear expressions $\alpha$,
  the following statements are equivalent for all $\rangeo{i}{0}{\lambda}$
  \begin{enumerate}
  \item\label{prop:dfs:1} $\M, i \models \df(\alpha)$
  \item\label{prop:dfs:2} $v_{h,i}(\next^l x) = v_{t,i}(\next^l x) \not = \undefined$ for all $\next^l x \in \vars(\alpha)$
  \end{enumerate}
\end{proposition}

In our context,
an \emph{assignment} for a variable $x\in \mathcal{X}$ is an expression of the form $\next^l x := \alpha $,
where $\alpha$ is a linear temporal term.
The assignment operator can be seen as a derived operator defined as
\begin{align*}
  \next^l x & := \alpha \eqdef  (\df(\alpha) \rightarrow \next^l  x = \alpha).
\end{align*}

With it,
we now define a syntactic subclass that takes the form of a logic program.
To ease notation, we let
\(
(\varphi \leftarrow \psi) \eqdef (\psi \to \varphi)
\)
and
\(
(\varphi,\psi) \eqdef (\psi \wedge \varphi)
\).
Then,
a \textit{temporal linear constraint rule} (or just \emph{rule}) is of the form:\begin{align} \label{tlc:rule}
  \next^l x & := \alpha \leftarrow \ell_1, \dots, \ell_m, \neg \ell_{m+1}, \dots, \neg \ell_{k}
\end{align}
where $\next^l x := \alpha$ is an assignment and each $\ell_i$ is a linear temporal constraint atom for $0 \le m \le k$.
Placing Rule~\eqref{tlc:rule} under the $\alwaysF$ operator ensures its satisfaction across all temporal states.
As with \HTC,
assignments appearing in rule heads can be transformed into an analogous logic programming representation.
\begin{proposition}\label{translation}
  A rule of form~\eqref{tlc:rule} is equivalent to the rule
  \begin{eqnarray*}
    \next^l  x = \alpha  &\leftarrow &\next^{o_1} x_1 \le \next^{o_1} x_1, \dots , \next^{o_n}x_n\le \next^{o_n}x_n ,
                          \ell_1, \dots, \ell_m, \neg \ell_{m+1}, \dots, \neg \ell_{k}.
  \end{eqnarray*}
  where $\next^{o_i}x_i\in \vars(\alpha)$ for $0 \le i \le n = \left|\vars(\alpha)\right|$.
\end{proposition}

The logic programming fragment of \THTC\ offers a powerful approach to modeling various aspects of dynamic systems,
including the representation of inertia rules and default values.
To demonstrate the expressive power of our formalism, let us reconsider our initial scenario:
``{\em
  A radar is positioned at the 400 km mark on a road with a speed limit of 90 km/h.
  A car is initially traveling at 80 km/h.
  At time instant 4, the car accelerates by 11.35 km/h.
  Subsequently, at time instant 6, it decelerates by 2.301 km/h.
  The problem is to determine whether the car will exceed the speed limit and thus incur a fine.}''

To formalize this scenario within \THTC, we introduce the following variables.
The numerical fluents $p$ and $s$ (ranging over $\mathbb{N}$) represent the car's position and speed, respectively.
The numerical variables $\mathit{rdpos}$ and $\mathit{rdlimit}$ (ranging over $\mathbb{N}$) refer to
the radar's position and speed limit.
The numerical action $\mathit{acc}$ (ranging over $\mathbb{Z}$) models the car's acceleration (positive) or deceleration (negative).
Lastly, the Boolean variable $\mathit{fine}$ indicates whether the car's speed exceeded $\mathit{rdlimit}$ at $\mathit{rdpos}$,
leading to a fine.

The subsequent set of rules formalizes our scenario in \THTC.\footnote{To facilitate understanding,
  the values of $s$, $\mathit{acc}$, and $\mathit{pos}$ are presented using decimal notation.
  In a real-world implementation, speed and acceleration should be expressed in $m/h$,
  while position should be expressed in meters.}
\begin{align}
	p & :=  0                                \label{rule:1}\\
	s & :=  80                               \label{rule:2}\\
	\Box (\mathit{rdlimit}& :=  90)          \label{rule:3}\\
	\Box (\mathit{rdpos} & := 400)           \label{rule:4}\\
	\Box (\next s & := s + acc)              \label{rule:s} \\
	\Box (\next s & := s  \leftarrow \neg (\next s \neq s))    \label{rule:inertia:s}\\
	\Box (\next p & := p + s )                \label{rule:pos}\\
	\Box (\next \mathit{fine} & \leftarrow  p < \mathit{rdpos}, \next p \ge \mathit{rdpos},\next s > \mathit{rdlimit}) \label{rule:fine}
\end{align}
Rules preceded by the always operator $\alwaysF$ in our \THTC\ formalization are enforced throughout the entire temporal evolution,
contrasting with Rules~\eqref{rule:1} and~\eqref{rule:2},
which are specific to the initial state.
Specifically, Rules~\eqref{rule:1} and~\eqref{rule:2} set the car's initial position to $0$ meters and its speed to $80~km/h$.
The radar's characteristics, its position at $400~km$ and a constant speed limit of $90~km/h$,
are defined by Rules~\eqref{rule:3} and~\eqref{rule:4}.
The dynamics of the car's speed are captured by Rules~\eqref{rule:s} and~\eqref{rule:inertia:s}:
the speed remains unchanged unless an acceleration value ($\mathit{acc}$) is present,
causing a corresponding adjustment.
Rule~\eqref{rule:pos} dictates that the car's position changes based on its speed.
The condition for receiving a fine is formalized in Rule~\eqref{rule:fine}:
if the car's speed ($s$) exceeds the radar's speed limit ($\mathit{rdlimit}$)
at or after passing the radar's position ($\mathit{rdpos}$) at any time,
a fine is issued.
Finally, to model the acceleration of $11.350~km/h$ at time $4$ and the deceleration of $2.301~km/h$ at time $6$,
we include the following temporal assignments:
\begin{align*}
  \next^{4} \mathit{acc} & :=  11.35 & \next^{6} \mathit{acc} &:= -2.301
\end{align*}
These two rules for acceleration and deceleration are applied specifically at time points $4$ and $6$, respectively.
Note that our formalization does not provide a value for $\mathit{acc}$ in the initial state.
Since no value is provided,
$\mathit{acc}$ is left undefined at the initial state,
Rule~\eqref{rule:s} is inapplicable
and the inertia rule in~\eqref{rule:inertia:s} keeps the speed ($\mathit{s}$) constant.
However, at time points $4$ and $6$,
a value for $\mathit{acc}$ is given,
so Rule~\eqref{rule:s} applies and the value of $\mathit{s}$ is modified.

An alternative formalization could rely on the use of the default rule
\begin{align*}
  \alwaysF \left(acc:=0 \leftarrow \neg \left(acc\not=0\right)\right)
\end{align*}
to set the value of $\mathit{acc}$ to $0$ in the absence of information.
Under the addition of such a rule,
inertia~\eqref{rule:inertia:s} could be removed, since it becomes redundant:
when there is no acceleration, $acc=0$ and we already derive $\next s=s$ through $\eqref{rule:s}$.

The following table presents an equilibrium model that satisfies the representation of our scenario.
The values for $\mathit{rdlimit}$ and $\mathit{rdpos}$ are omitted from the table as they remain constant over time.
The evolution of the remaining fluents is as follows:
\begin{center}
	\begin{tabular}{cccc}
		\textit{time}   &        $\mathit{s}$ $(km/h)$ &         $\mathit{p}$ $(km)$&           $\mathit{acc}$ $(km/h)$\\
		$0$  &            $80$  &                 $0$     &            $\undefined$  \\
		$1$  &            $80$  &                 $80$    &            $\undefined$  \\
		$2$  &            $80$  &                 $160$   &            $\undefined$  \\
		$3$  &            $80$  &                 $240$   &            $\undefined$  \\
		\underline{$4$} & \underline{$80$} &       \underline{$320$} &    \underline{$11.35$}      \\
		\underline{$5$} & \underline{$91.35$} &    \underline{$400$} & \underline{$\undefined$} \\
		$6$  &            $91.35$  &             $491.35$  &             $-2.301$  \\
		$7$  &            $89.049$  &            $582.7$  &            $\undefined$  \\
		$8$  &            $89.049$  &            $671.749$  &            $\undefined$
	\end{tabular}
\end{center}
Between time points $4$ and $5$,
the car's speed exceeds the limit as it passes the radar's position,
resulting in the $\mathit{fine}$ atom becoming true at time point $5$.

 \section{Discussion}\label{sec:discussion}

We have integrated two established extensions of \HT, namely \THT\ and \HTC\ (along with their equilibrium variants),
into a unified framework for temporal (nonmonotonic) reasoning with constraints.
Our work is inspired by extensions of (monotonic) Linear Temporal Logic (\LTL)
interpreted over constraint systems and first-order theories,
in particular the approaches taken by~\cite{demri06a} and \cite{gegigi22a},
which apply next and previous operators to non-Boolean variables.
In many of these cases, the constraints considered are based on Presburger arithmetic~\citep{demri06b} or
qualitative spatial formalisms~\citep{vilkau86a,wolzak00a,balcon01a}.
Within the ASP paradigm,
a significant early extension of temporal answer sets with constraints was introduced by \cite{gimaspd13a},
building upon their earlier work~\citep{gimath13a}.
This approach utilizes logic programming syntax and semantics,
which differs from our \HT-based approach.
Furthermore, their constraints typically refer to variables within individual states,
unlike our approach which allows for the combination of variables across multiple temporal states.

In addition to defining and formally elaborating on the semantics of \THTC,
we extended Kamp's translation~\citep{kamp68a} to our approach by mapping \THTC\ formulas into \QHT,
and notably demonstrated that this correspondence extends to the respective equilibrium models.
We further investigated the logic programming fragment of \THTC,
illustrating its expressive power through modeling our initial scenario involving temporal numeric constraints.
This work lays the groundwork for the future integration of constraint reasoning into other \HT-based temporal
extensions, such as dynamic and metric logics of \HT\ (and their equilibrium variants)~\citep{bocadisc18a,becadiscsc24a}.
Advancing this line of research will necessitate adapting existing computational methods~\citep{becadiharosc24a}
to operate over constraints, extending beyond solely propositional atoms.

As a first step, we have only considered linear constraints.
Future work includes studying the integration of periodicity constraints~\citep{demri06b} and
qualitative spatial constraints~\citep{wolzak00a,balcon01a}.
Periodicity constraints would allow for expressing congruence relations among different variables,
proving particularly useful for capturing cyclical behaviors or recurring events.
Integrating spatial constraints would enable us to assign spatial meaning to variables and describe their dynamics,
for instance, modeling objects that change their position or size over time.

In the classical setting,
the computational complexity of temporal logics interpreted over constraint systems is often highly undecidable~\citep{demri06a},
with decidability strongly depending on the specific constraint system.
For example,
the satisfiability problem for \LTL\ interpreted over constraint systems based on Presburger arithmetic (including periodicity constraints) is in PSPACE under certain completion properties~\citep{demri06a,balcon02a}.
When considering qualitative spatial constraints expressed using the Region Connection Calculus,
the problem is decidable,
with complexity ranging from NP to EXPSPACE, depending on the imposed restrictions~\citep{wolzak00a}.
Following a strategy similar to~\citep{cabdem11a},
it may be possible to establish a bijection between classical and equilibrium models,
potentially allowing us to establish a lower bound on the complexity of the satisfiability problem in \THTC.
However, determining a corresponding upper bound remains an open challenge.

 \section*{Acknowledgments}
This work was supported by the DFG grant SCHA 550/15 (Germany), the Etoiles Montantes CTASP project (Region Pays de la Loire, France), and the MICIU/AEI/10.13039/501100011033 grant PID2023-148531NB-I00 (Spain).

\newpage
\appendix
\section{Proofs: For reviewing purposes}\label{sec:proofs}

\begin{proposition} the relation $\sqsubseteq$ is a partial order relation.
\end{proposition}	
\begin{proof}  We prove that the relation $\sqsubseteq$ is reflexive, antisymmetric and transitive:
	\begin{itemize}[itemsep=0pt]
		\item \textit{Reflexivity}: Assume by contradiction that $v \not \sqsubseteq v$, therefore there exists $ (x,d) \in \mathcal{X} \times \mathcal{D}$ s.t. $v(x)=d$ and $v(x) \not = d$, which is a contradiction. 
		\item \textit{Transitivity}: Assume by contradiction that $v_1 \sqsubseteq v_2$ and $v_2 \sqsubseteq v_3$ but $  v_1 \not \sqsubseteq v_3$. Then there exists $ (x,d) \in \mathcal{X} \times \mathcal{D}$ s.t. $v_1(x) =\mathcal{D}$ and $v_3(x) \not =d$. Since $v_1 \sqsubseteq v_2$, then $v_2(x) =d$, and since $v_2  \sqsubseteq v_3$, then $v_3(x) =d$, which is a contradiction. 
		\item \textit{Antisymmetry}: Assume by contradiction that $v \sqsubseteq v'$ and $v' \sqsubseteq v$  but $v \not = v'$. It follows, without loss of generality, that there exists $ (x,d) \in \mathcal{X} \times \mathcal{D}$ s.t. $v(x)=d$ and $v'(x) \not = d$. Since $v \sqsubseteq v'$ then $v'(x) = d$, which is a contradiction.
	\end{itemize}	
\end{proof}	

\begin{proofof}{Proposition~\ref{prop:persistence}} By structural induction. Let us take $\M=\tuple{\Vh,\Vt}$.
	For the base case (a constraint atom $c(\next^{o_1}x_1,\dots,\next^{o_n}x_n)$) the proof goes as follows:

$\M, i \models c(\next^{o_1}x_1,\dots,\next^{o_n}x_n) $ iff $ (v_{h,i}(\next^{o_1} x_1),\dots,v_{h,i}(\next^{o_n} x_n)) \in \tbl{c(\next^{o_1}x_1,\dots,\next^{o_n}x_n)}$ and $ (v_{t,i}(\next^{o_1} x_1),\dots,v_{t,i}(\next^{o_n} x_n)) \in \tbl{c(\next^{o_1}x_1,\dots,\next^{o_n}x_n)}$. Therefore,  $\tuple{\Vt,\Vt},i \models c(\next^{o_1}x_1,\dots,\next^{o_n}x_n) $.
We present below the proof for the following connectives:

\begin{itemize}[itemsep=0pt]
	\item Case $\varphi \land \psi$: $\M,i \models \varphi \land \psi$ iff $\M,i \models \varphi $ and $\M,i \models  \psi$. By induction, $\tuple{\Vt,\Vt}, i \models \varphi$ and $\tuple{\Vt,\Vt}, i \models \psi$ iff $\tuple{\Vt,\Vt}, i \models \varphi \land \psi$.
	\item The case of the disjunction is similar to the case of the conjunction.
	\item Case $\varphi \to \psi$: it follows the semantics of implication.
	\item Case $\next \varphi$: if $\M, i \models \next \varphi$ then $i < \lambda -1$ and $M, i+1 \models \varphi$. By induction, $\tuple{\Vt,\Vt}, i+1 \models \varphi$ and therefore $\tuple{\Vt,\Vt}, i \models \next \varphi$.
	\item $\varphi \until \psi$: if $\M, i \models \varphi \until \psi$ then there exists $k \ge i$ s.t. $\M, k \models \psi$ and for all $i \le j < k$, $\M, j \models \varphi$. By induction, there exists $k \ge i$ s.t. $\tuple{\Vt,\Vt}, k \models \psi$ and for all $i \le j < k$, $\tuple{\Vt,\Vt}, j \models \varphi$. Therefore, $\tuple{\Vt,\Vt}, i \models \varphi \until \psi$.
	\item $\varphi \release \psi$: if $\M, i \models \varphi \release \psi$ then for all $k \ge i$ either $\M, k \models \varphi$ or $\M, j \models \varphi$ for some $i \le j < k$. By induction, for all $k \ge i$ either $\tuple{\Vt,\Vt}, k \models \varphi$ or $\M, j \models \varphi$ for some $i \le j < k$. Therefore, $\tuple{\Vt,\Vt}, i \models \varphi \release \psi$.
	\item $\previous \varphi$: if $\M, i \models \previous \varphi$ then $i > 0$ and $M, i-1 \models \varphi$.
   By induction, $\tuple{\Vt,\Vt}, i-1 \models \varphi$ and therefore $\tuple{\Vt,\Vt}, i \models \previous \varphi$.
   \item The proof for the operators $\since$ and $\trigger$ is done in a similar way.
\end{itemize}
\end{proofof}

\begin{proofof}{Corollary~\ref{cor:negation}}
We take $\M=\tuple{\Vh,\Vt}$. From left to right, if $\M, i \models \neg \varphi$ then, by Proposition~\ref{prop:persistence}, $\tuple{\Vt,\Vt}, i \models \neg \varphi$. Therefore, $\tuple{\Vt,\Vt}, i \not \models  \varphi$.
From right to left, if $\tuple{\Vt,\Vt},i \not \models \varphi$, then by Proposition \ref{prop:persistence}, $\tuple{\Vh,\Vt},i \not \models \varphi$. From the satisfaction relation it follows that $\M,i  \models \neg \varphi$.
\end{proofof}

\begin{proofof}{Proposition~\ref{prop:htc2thtc}} Proof sketch. The proof is done by structural induction.
	The case of a constraint atom is proved by using the correspondence between the denotation and the solution table.
	For the rest of the connectives we use both \HTC{} and \THTC{} satisfaction relation and the induction hypothesis.
\end{proofof}

\begin{proposition}\label{prop:tht2thtc:bijection} $\delta$ is a bijective function.
\end{proposition}
\begin{proofof}{Proposition~\ref{prop:tht2thtc:bijection}}
We first prove that $\delta$ is injective. Let us assume towards a contradiction that there exist two \HT{} traces $\M = \tuple{\Htrace,\Ttrace}$ and $\M'= \tuple{\Htrace',\Ttrace'}$ such that $\delta(\M)=\tuple{\Vh,\Vt} =\delta(\M')=\tuple{\Vh',\Vt'} $ but $\M \not = \M'$. Without loss of generality, let us assume that there exists $\rangeco{i}{0}{\lambda}$ such that $H_i \not = H'_i$. Again, without loss of generality, let us assume that there exists a propositional variable $p$ such that $p \in H_i$ but $p\not \in H'_i$. By definition, $v_{h,i}(p) = \true$ but $v'_{h,i}(p) = \undefined$. Therefore $\delta(\M) \not = \delta(\M')$: a contradiction.

We prove now that $\delta$ is surjective. Given an \HTC\ trace $\tuple{\Vh,\Vt}$ of length $\lambda$,
we define the corresponding \HT\ trace $\M = \tuple{(H_i)_{\rangeco{i}{0}{\lambda}},(T_i)_{\rangeco{i}{0}{\lambda}}}$ of length $\lambda$ where
\begin{align*}
	H_i & = \{ p \mid p \in \var \text{ and } v_{h,i}(p) = \true \} &
	T_i & = \{ p \mid p \in \var \text{ and } v_{t,i}(p) = \true \}.
\end{align*}

The reader can check that $\delta(\M) = \tuple{\Vh,\Vt}$.
\end{proofof}

\begin{proofof}{Proposition~\ref{prop:tht2thtc}} The proof is made by induction on the complexity of $\varphi$.
	For the sake of readability, let us fix $\delta(\tuple{\Htrace,\Ttrace}) = \tuple{\Vh,\Vt}$.
	\begin{itemize}[itemsep=0pt]
		\item Case of an atom $p$: from left to right, if $\tuple{\Htrace,\Ttrace},i \models p$, it implies that $p \in H_i \subseteq T_i$, so $p \in T_i$. 
		Because of the model correspondence defined between $\tuple{\Htrace,\Ttrace}$ and $\tuple{\Vh,\Vt}$, $v_{h,i}(p) = 1$ and $v_{t,i}(p) = 1$.
		Therefore, $v_{h,i}(p), v_{t,i}(p)  \in \tbl{p=\true}$, so $\tuple{\Vh,\Vt}, i \models p=\true$.
		From right to left, if $\tuple{\Vh,\Vt}, i \models p=\true$ then $v_{h,i}(p), v_{t,i}(p) \in \tbl{p=\true}$.
		By definition, $v_{h,i}(p) = v_{t,i}(p) = \true$.
		Because of the model correspondence, $p \in H_i$ and $p \in T_i$. By the semantics, $\tuple{\Htrace,\Ttrace},i \models p$.
	\item Disjunction and conjunction are obtained directly by induction.
	\item Case $\varphi \to \psi$:  from left to right, if $\tuple{\Htrace,\Ttrace},i \models \varphi \to \psi$, it means that (i) $\tuple{\Htrace,\Ttrace},i \not \models \varphi$ or $\tuple{\Htrace,\Ttrace},i  \models \psi$  and (ii) $\tuple{\Ttrace,\Ttrace},i \not \models \varphi$ or $\tuple{\Ttrace,\Ttrace},i  \models \psi$. Since $\tuple{\Ttrace,\Ttrace}$ is also a \THT{} model, we can apply the induction hypothesis on both (i) and (ii) in order to get (i') $\tuple{\Vh,\Vt},i \not \models \varphi$ or $\tuple{\Vh,\Vt},i  \models \psi$, and that (ii') $\tuple{\Vt,\Vt},i \not \models \varphi$ or $\tuple{\Vt,\Vt},i  \models \psi$. By the semantics it follows that
	$\tuple{\Vh,\Vt},i \models \varphi \to \psi$. The converse direction is proved in a similar way.
	\item Case $\next \varphi$: from left to right, if $\tuple{\Htrace,\Ttrace},i \models \next \varphi$,  it implies that $i < \lambda -1$ and $\tuple{\Htrace,\Ttrace},i+1 \models \varphi$. By induction, it follows that $\tuple{\Vh,\Vt},i +1 \models \varphi$, so $\tuple{\Vh,\Vt},i \models \next \varphi$ because of the satisfaction relation. A similar argument applies to the converse direction.
	\item The case of $\previous \varphi$ follows a similar reasoning as for the next operator.
	\item The proof for $\varphi \until \psi$,  $\varphi \release \psi$, $\varphi \since \psi$ and  $\varphi \trigger \psi$ can be done by using the induction hypothesis.
	\end{itemize}
\end{proofof}

\begin{proofof}{Proposition~\ref{prop:tht2thtc:equilibrium}}
	Since $\delta$ is bijective, there exists a one-to-one correspondence among the interpretations.
	Therefore, there is also a one-to-one correspondence among the order relations in \THT{} and \THTC{}. 
\end{proofof}

\begin{proofof}{Proposition~\ref{prop:derived:op}} It can be proved by only using the satisfaction relation. Left to the reader.
\end{proofof} 
\begin{proofof}{Proposition~\ref{lem:kamp}}
	Proof sketch. By structural induction. 
	\begin{itemize}
		\item Case $c(\next^{o_1}x_1,\dots,\next^{o_n}x_n)$: from left to right, if $\M, i \models c(\next^{o_1}x_1,\dots,\next^{o_n}x_n)$ then
		\begin{displaymath}
		(v_{h,i+o_1}(x_1),\cdots, v_{h,i+o_n}(x_n)),\; (v_{t,i+o_1}(x_1),\cdots, v_{t,i+o_n}(x_n)) \in \tbl{c(\next^{o_1}x_1,\dots,\next^{o_n}x_n)}.	
		\end{displaymath} 
		By the model correspondence relation we get $\sigma_h(f_{x_j}(i+o_j)) = v_{h,i+o_j}(x_j)$ for all $\rangeco{j}{1}{n}$.
		
		Since $(v_{h,i+o_1}(x_1),\cdots, v_{h,i+o_n}(x_n))\in \tbl{c(\next^{o_1}x_1,\dots,\next^{o_n}x_n)}$, by definition,
		$c(\sigma_h(f_{x_1}(i+o_1)), \cdots, \sigma_h(f_{x_n}(i+o_n))) \in I_h$.
		Since $(v_{t,i+o_1}(x_1),\cdots, v_{t,i+o_n}(x_n))\in \tbl{c(\next^{o_1}x_1,\dots,\next^{o_n}x_n)}$, by definition,
		$c(\sigma_t(f_{x_1}(i+o_1)), \cdots, \sigma_t(f_{x_n}(i+o_n))) \in I_t$.
		Therefore, $\mathcal{M} \models \ST{t}{c(\next^{o_1}x_1,\dots,\next^{o_n}x_n)}[t\leftarrow i]$.

		Conversely, let us assume that  $\mathcal{M}\models\ST{t}{c(\next^{o_1}x_1,\dots,\next^{o_n}x_n)}[t\leftarrow i]$.
		Therefore, $c(\sigma_h(f_{x_1}(i+o_1)), \cdots, \sigma_h(f_{x_n}(i+o_n)))\in I_h$.
		Since $(\sigma_h, I_h) \peq (\sigma_t,I_t)$, $c(\sigma_t(f_{x_1}(i+o_1)), \cdots, \sigma_t(f_{x_n}(i+o_n)))\in I_t$. 

		Because of the correspondence, we get that $(v_{h,i+o_1}(x_1), \cdots, v_{h,i+o_n}(x_n))$ and $(v_{t,i+o_1}(x_1), \cdots, v_{t,i+o_n}(x_n))$ belong to $\tbl{c(\next^{o_1}x_1,\dots,\next^{o_n}x_n)}$. 
		Therefore, $\M, i \models c(\next^{o_1}x_1,\dots,\next^{o_n}x_n)$.
		\item The proof for disjunction, conjunction and implication comes directly from the induction hypothesis.
		\item Case $\next \varphi$: from left to right, if $\M,i \models \next \varphi$ then $i < \lambda -1$ and $\M,i+1\models \varphi$. By induction hypothesis, $\mathcal{M}\models \ST{t}{\varphi}[t\leftarrow i+1]$.
		By some reasoning in first-order \HT{} it follows that $\mathcal{M} \models \exists i'\; \left(i'=i+1\right)\wedge \ST{i'}{\varphi}$, which means that $\mathcal{M}\models \ST{t}{\next \varphi}[t \leftarrow i]$.
		From right to left, if $\mathcal{M}\models \ST{t}{\next \varphi}[t \leftarrow i]$ then $\mathcal{M}\models \exists i'\; \left(i'  = i+1\right)\wedge \ST{i'}{\varphi}$, so $\mathcal{M}\models \ST{i'}{\varphi}[i'\leftarrow i+1]$. By induction hypothesis, $\M, i+1 \models \varphi$ so $\M, i \models \next \varphi$ by the \THTC{} satisfaction relation. 
		
		\item Case $\previous \varphi$: from left to right, if $\M,i \models \previous \varphi$ then $0 <i$ and $\M,i-1\models \varphi$. By induction hypothesis, $\mathcal{M}\models \ST{t}{\varphi}[t \leftarrow i-1]$.
		By some reasoning in first-order \HT{} it follows that $\mathcal{M} \models \exists i'\; \left(i'=i-1\right)\wedge \ST{i'}{\varphi}$, which means that $\mathcal{M}\models \ST{t}{\previous \varphi}[t \leftarrow i]$.
		From right to left, if $\mathcal{M}\models \ST{t}{\previous \varphi}[t \leftarrow i]$ then $\mathcal{M}\models \exists i'\; \left(i'  = i-1\right)\wedge \ST{i'}{\varphi}$, so $\mathcal{M} \models \ST{i'}{\varphi}[i'\leftarrow i-1]$. By induction hypothesis, $\M, i-1 \models \varphi$. Since $i>0$, $\M, i \models \previous \varphi$.
		\item Case $\varphi\until\psi$: From left to right, let us assume that $\M, i \models \varphi\until\psi$ then there exists $j \ge i$ s.t. $\M, j \models \psi$ and $\M, k\models \varphi$ for all $\rangeco{k}{i}{j}$. 
		By induction hypothesis it follows that $\mathcal{M} \models \ST{t}{\psi}[t \leftarrow j]$ for some $j \ge i$ and $\mathcal{M}\models \ST{t}{\varphi}[t\leftarrow k]$ for all $\rangeco{k}{i}{j}$.
		By reasoning on first-order logic, it follows that $\mathcal{M}\models \ST{t}{\varphi \until \psi}[t \leftarrow i]$. The converse direction is proved similarly.
		\item Case $\varphi\release\psi$: the proof is similar to the proof of until.
		\item Case $\varphi\since \psi$: the proof is similar to the one for the until but reasoning on the past instead.
		\item Case $\varphi\trigger\psi$: the proof is similar to the proof of since.
	\end{itemize}
\end{proofof}

\begin{proofof}{Proposition~\ref{lem:kamp:equilibrium}}
We begin by noting that the ordering relation between \HTC{} traces  $\Vh \le \Vt$ is also in a one-to-one correspondence with the ordering relation
$(\sigma_h,I_h) \peq (\sigma_t,I_t)$. Then, the result follows from the definitions of equilibrium models together with Lemma~\ref{lem:kamp}.
\end{proofof}	

 \begin{proofof}{Proposition~\ref{prop:complement}}
	
	\noindent If $\M, i \models c(\next^{o_1}x_1,\cdots,\next^{o_n}x_n)$ then both $(v_{h,i+o_1}(x_1),\cdots,v_{h,i+o_n}(x_n)) \in \tbl{c(\next^{o_1}x_1,\cdots,\next^{o_n}x_n)}$
	and $(v_{t,i+o_1}(x_1),\cdots,v_{t,i+o_n}(x_n)) \in \tbl{c(\next^{o_1}x_1,\cdots,\next^{o_n}x_n)}$. 
	By the definition of $\tbl{\overline{c}(\next^{o_1}x_1,\cdots,\next^{o_n}x_n)}$ it follows that
	$(v_{h,i+o_1}(x_1),\cdots,v_{h,i+o_n}(x_n))\not \in \tbl{\overline{c}(\next^{o_1}x_1,\cdots,\next^{o_n}x_n)}$ and 
	$(v_{t,i+o_1}(x_1),\cdots,v_{t,i+o_n}(x_n)) \not \in \tbl{\overline{c}(\next^{o_1}x_1,\cdots,\next^{o_n}x_n)}$. 	
	By the satisfaction relation we conclude that $\M, i \not \models \overline{c}(\next^{o_1}x_1,\cdots,\next^{o_n}x_n)$ and 
	$\tuple{\Vh,\Vt}, i \not \models \overline{c}(\next^{o_1}x_1,\cdots,\next^{o_n}x_n)$.
	By Corollary~\ref{cor:negation}, $\M, i \models \neg \overline{c}(\next^{o_1}x_1,\cdots,\next^{o_n}x_n)$.	
\end{proofof}

\begin{proofof}{Proposition~\ref{prop:dfs}} Let us assume without loss of generality that the variables in $\alpha$ range on $\mathbb{Z}$.
	Before presenting the proof we want to remark that 
	\begin{displaymath}
		\tbl{\next^{o_j}x_j \le \next^{o_j}x_j} = \lbrace a \in \mathbb{Z}\mid a \le a\rbrace = \mathbb{Z}.
	\end{displaymath}
\noindent with $\next^{o_j}x_j \in \vars(\alpha)$. The proof is presented below.
 \noindent If $\M, i \models \df(\alpha)$ then 
	\begin{equation*}
			\M, i \models \bigwedge\limits_{\next^{o_j} x_j \in \vars(\alpha)} \next^{o_j} x_j\le \next^{o_j} x_j.
		\end{equation*} 
From this we conclude that $\M,i \models\next^{o_j} x_j\le \next^{o_j} x_j$ iff $\rangeco{i+o_j}{0}{\lambda}$ and 
$\undefined \not = v_h^{i+o_j}(x_j)  \in \mathbb{Z}$.
Since $v_h^{i+o_j}(x_j) \not = \undefined$ and $v_{h,i+o_j}\sqsubseteq v_{t,i+o_j}$, $\undefined \not= v_{h,i+o_j} = v_{t,i+o_j}\in \mathbb{Z}$.
Since the constraint atom $\next^{o_j} x_j\le \next^{o_j} x_j$ was chosen arbitrary, we can conclude~\ref{prop:dfs:2}.
For the converse direction, let $\next^{o_j} x_j \in \vars(\alpha)$. 
If $\undefined \not = v_{h,i}(\next^{o_j} x_j) = v_{t,i}(\next^{o_j} x_j)$ then $\rangeco{i+ o_j}{0}{\lambda}$ and $v_{h,i+o_k}(x_j)= v_{t,i+o_k}(x_j) \in \mathbb{Z}$.
Therefore, $\M, i \models \next^{o_j} x_j\le \next^{o_j} x_j$. Since $\next^{o_j} x_j\in \vars(\alpha)$ was chosen arbitrary, it follows that 
$\M, i \models \next^{o_j} x_j\le \next^{o_j} x_j$, for all $\next^{o_j}x_j \in \vars(\alpha)$. Therefore,~\ref{prop:dfs:1} holds.
\end{proofof}

\begin{proofof}{Proposition~\ref{translation}} In this proof we consider the language of propositional logic. 
	\begin{equation}
		B_1\wedge  \cdots\wedge B_m\wedge \neg B_{m+1} \wedge \cdots \wedge \neg B_{k} \rightarrow \next^l x := \alpha
	\end{equation}
	\noindent is equivalent to 
	\begin{equation*}
		\left(B_1\wedge  \cdots\wedge B_m\wedge \neg B_{m+1} \wedge \cdots \wedge \neg B_{k}\right) \rightarrow \left(\df(\alpha)\rightarrow \next^l x = \alpha\right).
	\end{equation*}
	Since a double implication formula $(\varphi \to (\psi \to \chi))$ is logically equivalent to  $(\varphi \land \psi) \to \chi$, the aforementioned formula is equivalent to  
	\begin{equation*}
		B_1\wedge  \cdots\wedge B_m\wedge \neg B_{m+1} \wedge \cdots \wedge \neg B_{k} \wedge \df(\alpha) \rightarrow \next^l x = \alpha.
	\end{equation*}
	
	\noindent By replacing $\df(\alpha)$ by its definition we get 
	
	\begin{equation*}
		\left(\phantom{\bigwedge\limits_{jj}}B_1\wedge  \cdots\wedge B_m\wedge \neg B_{m+1} \wedge \cdots \wedge \neg B_{k} \wedge
		\bigwedge\limits_{\next^{o_i}x_i\in \vars(\alpha)}\next^{o_i}x_i \le \next^{o_i}x_i\right)\rightarrow \next^l x = \alpha.
	\end{equation*}
\end{proofof}  
\end{document}